\documentclass{article}

\usepackage{PRIMEarxiv}

\usepackage[utf8]{inputenc} 
\usepackage[T1]{fontenc}    
\usepackage{hyperref}       
\usepackage{url}            
\usepackage{booktabs}       
\usepackage{amsfonts}       
\usepackage{nicefrac}       
\usepackage{microtype}      
\usepackage{lipsum}
\usepackage{fancyhdr}       
\usepackage{graphicx}       
\usepackage{xcolor}         
\usepackage{wrapfig}
\usepackage{amsmath}
\usepackage{amssymb}
\usepackage{amsthm}
\usepackage{dsfont}
\usepackage{xurl}
\usepackage{txfonts}
\usepackage{subfigure}
\usepackage{tikz}
\usetikzlibrary{3d}
\usepackage{url}
\usepackage[numbers]{natbib} 
\usepackage{algorithm}
\usepackage{algorithmic}
\newcommand{\operator}{\mathcal{T}}

\newtheorem{theorem}{Theorem}[section]
\newtheorem{proposition}[theorem]{Proposition}
\newtheorem{lemma}[theorem]{Lemma}
\newtheorem{corollary}[theorem]{Corollary}
\theoremstyle{definition}

\newtheorem{remark}[theorem]{Remark}
\newtheorem{assumption}[theorem]{Assumption}

\DeclareMathOperator*{\argmin}{arg\,min}
\graphicspath{ {figs/} }

\title{Simulation-Based Optimistic Policy Iteration For Multi-Agent MDPs with Kullback-Leibler Control Cost}

\author{
  Khaled Nakhleh\\
  \vspace{0.25em} \\
  Electrical and Computer Engineering Department \\
  Texas A\&M University \\
  College Station, TX\\
  \texttt{\{khaled.jamal\}@tamu.edu} \\
   \And
  Ceyhun Eksin  \\
  Industrial and System Engineering Department \\
  Electrical and Computer Engineering Department \\
  Texas A\&M University \\
  College Station, TX\\
  \texttt{eksinc@tamu.edu} \\
   \And
  Sabit Ekin \\
  Engineering Technology and Industrial Distribution Department \\
  Electrical and Computer Engineering Department \\
  Texas A\&M University \\
  College Station, TX\\
  \texttt{sabitekin@tamu.edu} \\
}

\begin{document}
\maketitle


\vspace{1em}
\begin{abstract} \label{sec:abstract}
This paper proposes an agent-based optimistic policy iteration (OPI) scheme for learning stationary optimal stochastic policies in multi-agent Markov Decision Processes (MDPs), in which agents incur a Kullback-Leibler (KL) divergence cost for their control efforts and an additional cost for the joint state. The proposed scheme consists of a greedy policy improvement step followed by an m-step temporal difference (TD) policy evaluation step. We use the separable structure of the instantaneous cost to show that the policy improvement step follows a Boltzmann distribution that depends on the current value function estimate and the uncontrolled transition probabilities. This allows agents to compute the improved joint policy independently. We show that both the synchronous (entire state space evaluation) and asynchronous (a uniformly sampled set of substates) versions of the OPI scheme with finite policy evaluation rollout converge to the optimal value function and an optimal joint policy asymptotically.  
Simulation results on a multi-agent MDP with KL control cost variant of the Stag-Hare game validates our scheme's performance in terms of minimizing the cost return.
\end{abstract}
\section{Introduction} \label{sec:introduction}
Consider the two-agent MDP version \cite{kappen2012optimal} of the Stag-Hare game \cite{skyrms2004stag} where two hunters move on a gridworld to hunt hares or a stag (see Figure~\eqref{fig:five_by_five_words} for an illustration).
Each hunter can only determine their own next position (local state) in this world by moving to a neighboring grid, but their costs are determined based on their joint position (state), and whether they hunt a stag or a hare. 
The hunters can individually hunt a hare, but need to coordinate together to be able to hunt a stag. 
Considering the returns are the same amongst the hunters, this setting is an example of a multi-agent Markov decision process (MDP), in which each agent has control of its local state but there is a common value function (identical interest) that depends on the joint state. 
\begin{figure}
\centering \includegraphics[width=0.4\textwidth, height=2.6in]{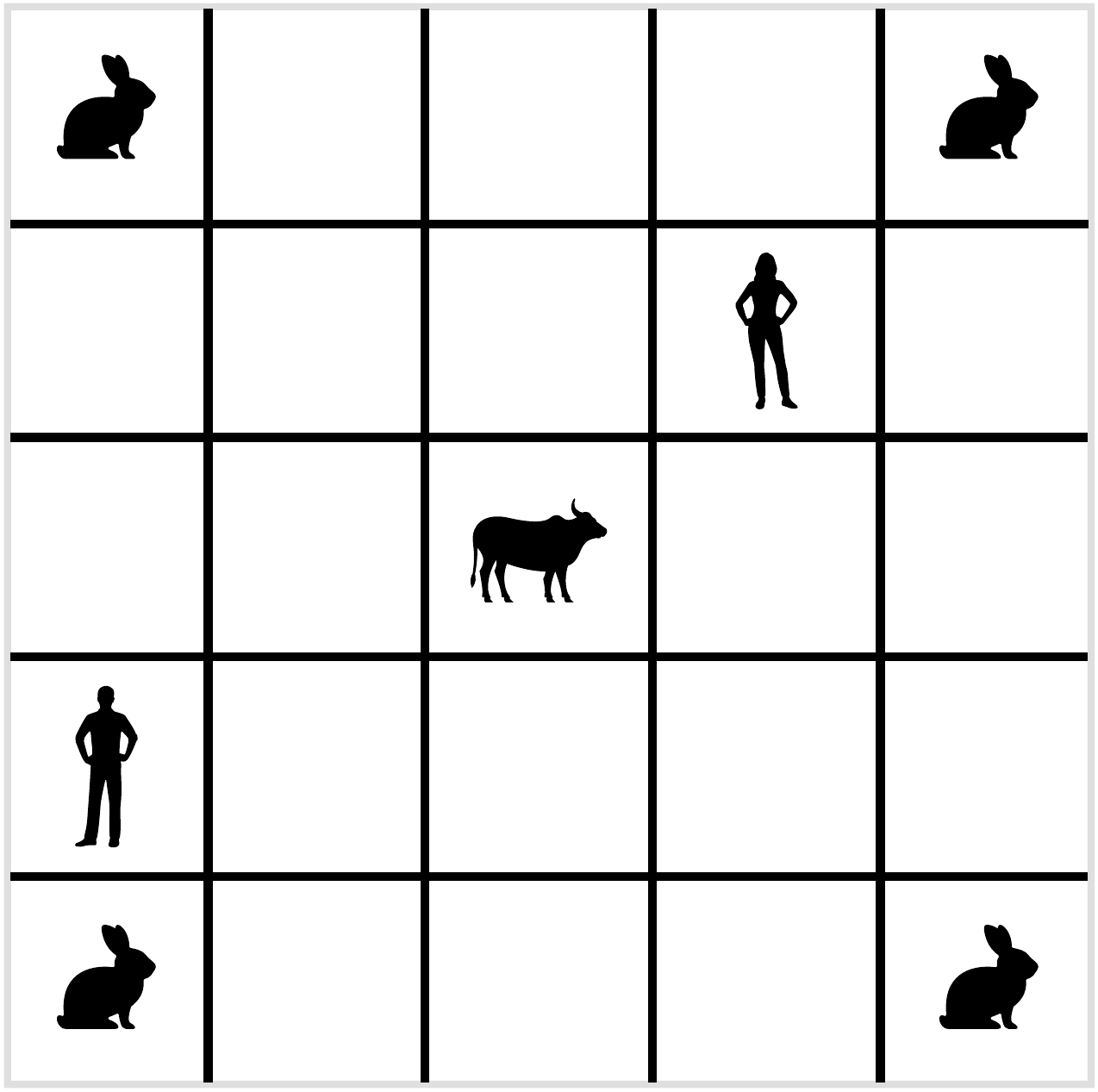}
\caption{A multi-agent MDP with KL control cost: Two hunters hunting either hares or a stag on a $5 \times 5$ gridworld. See Section \eqref{section:simulations} for details.}
\label{fig:five_by_five_words}
\vspace{-12pt}
\end{figure}

We assume the instantaneous costs are decomposed into two parts: an identical cost term that only depends on the joint state, and another term that captures the cost of control. 
The cost of control is measured by the KL divergence between the uncontrolled transition probabilities and transition probabilities collectively chosen by the agents. 
The KL cost represents the control cost the agents are willing to pay in order to modify the uncontrolled transition probability function. 
We coin the class of MDPs we consider as the multi-agent MDPs with Kullback-Leibler (KL) control cost which stem from the linearly solvable MDPs \cite{todorov2006,todorovpnas} and the adversarial linearly solvable Markov games \cite{dvijotham2012linearly}.

In the proposed Kullback-Leibler controlled optimistic policy iteration ($\texttt{KLC-OPI}$) scheme, each agent straddles between policy improvement and policy evaluation steps as is the case with the single agent OPI \cite{puterman1978}.
Optimistic policy iteration generalizes value iteration and policy iteration methods by considering a finite $m$-step temporal difference (TD) rollout in the policy evaluation step, where we obtain the value iteration method when $m=1$ and the policy iteration method when $m \rightarrow \infty$.
In the policy improvement step, the agents update their policies greedily considering their current policy valuations. 
The updated policies are then evaluated using the finite $m$-step TD rollout with sampled trajectories, where agents concurrently generate trajectories for each joint state in the joint state space, and compute the discounted returns.
The $m$-step TD rollout allows for a less noisy unbiased estimation of the policy compared to value iteration, while performing a finite rollout remains practically viable in contrast to policy iteration.

The instantaneous cost decomposition along with the infinite horizon discounted sum of the costs render a close form for the greedy optimal policies, that follow a Boltzmann distribution using a Cole-Hopf transformation of the value function (Lemma~\ref{lemma:joint_policy}). 
The resultant controlled transition probabilities selects next states inversely proportional to the exponent of the states' cost returns. 
That is, the KL control cost minimization solution results in a stochastic policy that only depends on the uncontrolled transition probability function and the current value function estimate, which eliminates the combinatorial search over the state-action space. 
Indeed, we do not need to restrict the action space to be finite, and allow the individual action spaces to be continuous. 
In addition, each agent can sample their next sub-state from its marginal controlled policy having computed the closed-form joint policy. 

For simulation-based policy iteration schemes, a synchronous policy evaluation step requires simulating a trajectory for each joint state in the state space.
We also consider the asynchronous implementation of the \texttt{KLC-OPI} where a subset of the joint state space is evaluated per iteration, potentially eliminating the need to run trajectories for each joint state.
We show that the asynchronous version of the $\texttt{KLC-OPI}$, namely $\texttt{ASYNC-KLC-OPI}$, can be related to the synchronous version of the scheme.
We analyze the asymptotic convergence of the two schemes, $\texttt{KLC-OPI}$ and $\texttt{ASYNC-KLC-OPI}$, and prove given standard assumptions on the learning rate and the initialized value functions, that the schemes' iterates asymptotically converge to the optimal value function for all agents and to an optimal joint policy (Theorem \ref{theorem:global-error-convergence} and Corollary~\ref{corollary:async-klc-opi-convergence}, respectively).
Finally, simulation results on the aforementioned multi-agent MDP variant of the Stag-Hare game show that the scheme is able to learn a joint policy that minimizes the cost return (Section~\eqref{section:simulations}).

\section{Related Work and Contributions} \label{sec:related-works}

\subsection{Simulation-Based Optimistic Policy Iteration}   
Optimistic, also known as simulation-based modified, policy iteration methods are preferred in settings with large state spaces due to their fast convergence  \cite{silver2017mastering}. However, the theoretical underpinnings of this practical success remain unclear. The pioneering theoretical guarantee for OPIs with Monte-Carlo estimation established convergence assuming infinitely long trajectories that start from each state at every iteration \cite{tsitsiklis2002}. Indeed, we follow the same proof structure here but consider finitely long trajectories. Similar, OPI schemes were considered for the stochastic shortest path problem in \cite{chen2018,liu2021}. More recent efforts focused on showing convergence of finite-step rollouts during the policy evaluation step after applying multi-step greedy or lookahead policies \cite{efroni2019,winnicki2022,winnicki2023convergence}. In particular, prior the asymptotic error bound was strengthened in \cite{winnicki2023convergence} using stochastic approximation techniques for the setting where policy evaluation is done using only a single Monte-Carlo trajectory in each iteration.
In contrast to these works, we prove asymptotic convergence to the optimal value function and an optimal joint policy without requiring a lookahead operation, i.e., a greedy improvement, while also performing a finite rollout in each iteration (Theorem \ref{theorem:global-error-convergence}). Our asymptotic convergence proof technique is applicable to any single-agent MDP with deterministic policies, not just MDPs with KL control costs.

\subsection{Decentralized Learning in Multi-Agent MDPs and Markov Games} 
The $\texttt{KLC-OPI}$ scheme considers multiple agents that perform the policy improvement step independently of other agents using an agent-local estimate of the value function (Lemma~\ref{lemma:joint_policy}). Thus, the proposed scheme falls under the framework of multi-agent reinforcement learning (MARL) \cite{zhang2021multi}, which has seen growing interest in the context of learning in identical interest or potential games~\cite{leonardos2021global,guo2023markov, unlu2023episodic}, zero-sum Markov games~\cite{sayin2021decentralized,sayin2022fictitious,brahma2022}, and multi-agent systems with KL control cost \cite{cammardella2023}. 
The considered multi-agent MDP framework is also equivalent to identical-interest Markov games or Markov cooperative games considered in \cite{leonardos2021global,ding2022independent}, where the costs/rewards are identical for all the agents.

Work on Markov games consider either policy gradients, in which agents consider parameterized policies updated using gradients computed through episodic returns \cite{leonardos2021global,aydin2023policy}, or a combination of standard learning protocols, e.g., best-response, fictitious, with a standard reinforcement learning algorithm, e.g., Q-learning \cite{sayin2022fictitious}.
Policy gradient methods are applicable to continuous state and action spaces' environments, but they suffer from large variance and the convergence rate being sensitive to the choice of parameters.

Recently, localized policy iteration methods for networked multi-agent systems \cite{zhang2023}  and zero-sum games \cite{bertsekas2021distributed} are shown to converge near globally optimal policy.  These schemes are based on state space partitioning with policy dependent mapping that gives a uniform sup-norm contraction property which enables convergence of the algorithms to the optimal value functions. 

The proposed $\texttt{KLC-OPI}$ scheme is a novel agent-based learning scheme that is shown to carry over the convergence properties of OPI designed for single-agent MDPs to the multi-agent MDP with KL control cost setting. The KL cost structure allows for continuous action spaces through the close-form solution to the policy improvement step.

\section{Multi-Agent MDPs With KL Control Cost} \label{sec:problem-setting}

We consider an infinite-horizon discounted $n$-agent MDP given as the tuple $\Gamma := \{\mathcal{N}, \mathcal{S}, \{\mathcal{A}\}_{i=1}^n,  P, \rho, C, \gamma\}$ with a finite number of players $\mathcal{N}:=\{1,\dots, n\}$, finite joint state space $\mathcal{S}$, and a continuous action space $\mathcal{A} = \mathcal{A}_{i=1} \times  \mathcal{A}_{i=2} \times \hdots \times \mathcal{A}_{i=n}$. 
The transition probability function $P:\mathcal{S} \times \mathcal{A} \to \Delta(\mathcal{S})$ determines the transition probability to the next joint state $s_{t+1}\in \mathcal{S}$ given the joint state $s_t \in \mathcal{S}$ and joint action profile $\mathbf{a}_{t}$ at time step $t\in \mathbb{N}^+$. 
The initially sampled joint state is chosen from a prior $\rho$, for which we assume $\rho(s)>0$ for all joint states $s\in \mathcal{S}$.

The intrinsic joint state cost function is defined as $C : \mathcal{S} \rightarrow \mathbb{R}$.

As done in~\cite{todorov2006} for single-agent linearly solvable MDPs, in multi-agent MDPs with KL control cost, the one-step cost function is composed of two terms: the intrinsic joint state cost function $C$, and the control cost, measured using the Kullback-Leibler (KL) divergence between the controlled and uncontrolled transition probability function.
Finally, $\gamma \in [0,1)$ is the discount factor.
\subsection{Stochastic Joint and Marginal Policies} \label{sec:multi-agent-joint-policy}
In the KL control setting, agents pick a joint policy by re-weighting $P$ with a continuous-valued action profile $\mathbf{a} \in \mathcal{A}$ that directly specifies the transition probability from $s$ to $s' \in \mathcal{S}$.
Hence, the agents avoid the combinatorial search over the state-action space for an action profile that is then applied to the transition probability function $P$.

\begin{assumption} \label{assumption:submodular} The joint state $s_t \in \mathcal{S}$ at time step $t$ is composed of $n$ sub-states where the sub-state $s_{i,t} \in \mathcal{S}_i$ can only be controlled by agent $i$ at time step $t$. The joint state is then $s_t = [s_{1,t}, s_{2,t}, \hdots, s_{n, t}]$.
\end{assumption}
Given Assumption~\ref{assumption:submodular}, the joint state space can be written as $\mathcal{S} = \mathcal{S}_1 \times \mathcal{S}_2 \times \hdots \times \mathcal{S}_n$.
Moreover, and similar to product games~\citep{flesch2008stochastic}, the multi-agent MDP with KL control cost transition structure is derived by taking the product of $n$ Markov transition structures.
Agent $i \in \mathcal{N}$ only controls their sub-state transitions through a probability transition function $P_i : \mathcal{S} \times \mathcal{A}_i \rightarrow \Delta(\mathcal{S}_i)$ such that $P = \times_{i \in \mathcal{N}} P_i$.
Agent $i$ obtains its uncontrolled transition probability function $P_{i,0} : \mathcal{S} \rightarrow \Delta(\mathcal{S}_i)$  by applying  action $a_i = 0$ such that the transition probability for any $s'_i \in \mathcal{S}_i$ and $s \in \mathcal{S}$

\begin{align}
    P_{i,0}(s'_i | s) = P_i(s'_i | s, a_i = 0).
\end{align}

\begin{assumption} \label{assumption:homogenous} Given $\Gamma$ and a joint state $s \in \mathcal{S}$, the uncontrolled transition probabilities of agents' sub-states are equal, i.e., $P_{i,0}(s_i' | s) = P_{j,0}(s_j' | s)$ for all $s_i' = s_j'$ sub-state pairs and for all agent pairs $i,j \in \mathcal{N}$.
\end{assumption}

Given Assumptions~\ref{assumption:submodular} and~\ref{assumption:homogenous}, the uncontrolled transition probability function $P_0: \mathcal{S} \rightarrow \Delta(\mathcal{S})$ can be written for any state $s\in \mathcal{S}$

\begin{align}
P_0(s' | s) = \prod\limits_{i=1}^{n} P_{i,0}(s_i' | s).
\end{align}

In this sense, when no re-weighting is made of $P$, i.e. $\mathbf{a} = \mathbf{0}$, we obtain the uncontrolled transition probability function $P_0$.
Note that, given Assumption~\ref{assumption:homogenous}, agent $i \in \mathcal{N}$ only requires the knowledge of $P_{i,0}$ for every joint state $s \in \mathcal{S}$ in order to compute $P_0$.

Let $\pi_{P_0} \in \Pi$ be the joint stochastic policy, or simply the joint policy, from the class of model-dependent stochastic policies $\Pi$ given the uncontrolled transition probability function $P_0$.
The joint policy is equal to a re-weighting of $P$ with a continuous-valued action profile such that for any $s, s' \in \mathcal{S}$, the weighted transition probability $ P(s' |s, \mathbf{a}) = \pi_{P_0}(s' | s)$. Hence, the joint policy given $P_0$ is the mapping $\pi_{P_0} : \mathcal{S} \rightarrow \Delta(\mathcal{S})$.

\begin{assumption} \label{assumption:equal-zero}
    Given any joint state pair $s, s' \in \mathcal{S}$, the joint policy $\pi_{P_0}(s' | s) = 0$ when $P_{0}(s' | s) = 0$.
\end{assumption}

Finally, the controlled sub-state transition is done using a stochastic marginal policy $\pi_{i,P_0} \in \Pi_{i}: \mathcal{S} \rightarrow \Delta(\mathcal{S}_i)$ that is derived from $\pi_{P_0}$ for all $i \in \mathcal{N}$.

\subsection{Decomposition of the One-Step Cost Function}
Similar to single-agent linearly solvable MDPs~\cite{todorov2006,todorovpnas}, the control cost is captured by the KL divergence between $\pi_{P_0}$ and the uncontrolled transition probability function $P_0$, i.e., 
\begin{align}
    D_{KL}\Big(\pi_{P_0}(\cdot | s) || P_0( \cdot | s) \Big) &:= \sum\limits_{s' \in \mathcal{S}} \pi_{P_0}( s' | s) \ln\Big(\frac{\pi_{P_0}(s' | s)}{P_0(s' | s)}\Big) =\mathop{\mathbb{E}}\limits_{s' \sim \pi_{P_0}(\cdot | s)} \Big[ \ln\Big(\frac{\pi_{P_0}(\cdot | s)}{P_0(\cdot | s)}\Big) \Big].
\end{align}
The one-step cost function is then written as $q(s, \pi_{P_0}) = C(s) + D_{KL}\big(\pi_{P_0}(\cdot | s) || P_0( \cdot | s) \big)$.
We then define the value function under $\pi_{P_0}$ to be state-wise
\begin{align} \label{equation:value-function-identical-interest-game}
    V^{\pi_{P_0}}(s) &:= \mathop{\mathbb{E}}\limits_{s' \sim \pi_{P_0}(\cdot | s)}\bigg[ \sum\limits_{t=0}^{\infty} \gamma^t \Big[C(s_t) +  D_{KL}\Big(\pi_{P_0}(\cdot | s_t) || P_0(\cdot | s_t)\Big) \Big]| s_{t=0} =s\bigg].
\end{align}
With $q_{max}$ being the maximum one-step cost, the value function is bounded by $q_{max}/ (1- \gamma)$.
Note that if the joint policy $\pi_{P_0}$ is equal to $P_0$, then the KL control cost is zero and the next joint state $s'$ is sampled according to $s' \sim P_0(\cdot | s)$, with $s' = [s_{1}', s_{2}', \hdots, s_{n}']$.
Given Assumption~\ref{assumption:equal-zero}, prohibitive transitions to another joint state in a single transition and the $D_{KL}$ control cost boundedness conditions are met.

\subsection{The KL Evaluation and KL Optimal Bellman Operators} 
We define the KL evaluation Bellman operator $\operator^{\pi_{P_0}}$ for $\pi_{P_0}$ applied state-wise to the value function as
\begin{align}  \label{equation:kl-bellman-equation}
    (\operator^{\pi_{P_0}}V)(s) &= C(s) + D_{KL}\Big(\pi_{P_0}(\cdot | s) || P_0(\cdot | s)\Big) + \gamma \sum\limits_{s'} \pi_{P_0}(s' | s)  V(s').
\end{align}
The KL {\it optimal} Bellman operator, denoted as $\operator$, is the Bellman operator given an optimal joint policy, i.e., $\operator V := \min\limits_{\pi_{P_0} \in \Pi} \operator^{\pi_{P_0}}V$. 
Then, the optimal value function $V^*$ satisfies $V^* = \operator V^*$.
Moreover, we denote the obtained optimal joint policy as $\pi^*_{P_0}$.
Both the KL evaluation and the KL optimal Bellman operators have the three properties: monotonicity, distributivity, and $\gamma-$contraction.
For any two arbitrary value functions $V_{1}, V_{2} \in \mathbb{R}^{|\mathcal{S}|}$ where $V_{1} \leq V_{2}$, $\operator^{\pi_{P_0}}$ exhibits monotonicity such that $\operator^{\pi_{P_0}} V_{1} \leq \operator^{\pi_{P_0}} V_{2}$.
In addition, for any constant $c_1 \in \mathbb{R},$ and an all-ones vector $\overline{\bf e}$, the operator $\operator^{\pi_{P_0}}$ has the distributivity property such that $\operator^{\pi_{P_0}}(V + c_1 \cdot \overline{\bf e}) = \operator^{\pi_{P_0}}V + \gamma \cdot c_1 \cdot \overline{\bf e}$.
A single application of the Bellman operator gives  $\gamma-$contractions in the $L_\infty$-Norm with $||\operator^{\pi_{P_0}}V_{1} - \operator^{\pi_{P_0}}V_{2}||_\infty \leq \gamma || V_{1} - V_{2} ||_\infty$.
\section{Multi-Agent Simulation-Based Kullback-Leibler Controlled Optimistic Policy Iteration $\texttt{KLC-OPI}$} \label{section:kl-opi-scheme}

\subsection{The Scheme Description}

The scheme is run by each agent $i \in \mathcal{N}$ on iterations $k=0,1,2,\hdots, K$ independently of other agents. 
At each iteration $k$ and for agent $i$, the scheme stores the value function estimate $V_{i,k}$ used in evaluating the joint policy.
Agent $i$ performs a greedy policy improvement step to obtain its independently calculated joint policy $\pi^{(i)}_{{P_0},k+1}$ as
\begin{align} \label{equation:greedy-step}
    \mathcal{G}(V_{i,k}) := \argmin\limits_{\pi^{(i)}_{P_0,k+1} \in \Pi} \operator^{\pi^{(i)}_{P_0,k+1}} V_{i,k} \subseteq \mathcal{S} \rightarrow \Delta(\mathcal{S}).
\end{align}

Given that the joint state can be decomposed into an agent $i \in \mathcal{N}$ sub-state $s_i$ as stated in Assumption~\ref{assumption:submodular}, agent $i$ only controls their transitions to a sub-state $s'_i \in \mathcal{S}_i$ using the marginal policy
\begin{align} \label{equation:marginal-policy}
\pi^{(i)}_{i,P_0,k+1}(s'_i | s) := \sum\limits_{j \in \mathcal{N} \backslash \{i\}}\sum\limits_{s'_j \in \mathcal{S}_j} \pi^{(i)}_{P_0,k+1}(s'_i, s'_j | s).
\end{align}
The agents collectively evaluate their value function estimates by generating $|\mathcal{S}|$ synchronous and coupled $m-$step TD sampled trajectories using their marginal policies \eqref{equation:marginal-policy} 
with each trajectory starting at a joint state $s \in \mathcal{S}$.
The term {\it synchronous} denotes evaluating the value function estimates for all joint states in the state space, while the term {\it coupled} means that the joint state transitions are done simultaneously by all agents.
The simulation-based policy evaluation step results in the per joint state cost function $q_t(s_t, \pi^{(i)}_{{P_0},k+1})$ for a timestep $t$ using the $m-$step TD trajectories. 
Each agent uses the discounted sum of the returns to obtain an estimate of $(\operator^{\pi^{(i)}_{P_0,k+1}})^m V_{i,k}$ which we represent using the noise term $\epsilon_{m,k}\in \mathbb{R}^{|\mathcal{S}|}$ as follows,
\begin{equation}\label{eq_estimate}
(\operator^{\pi^{(i)}_{P_0,k+1}})^m V_{i,k} + \epsilon_{m,k}=\sum_{t=0}^{m-1} \gamma^t q_t(s_t, \pi^{(i)}_{{P_0},k+1})+ \gamma^m V_{i,k}(s_{t=m}).
\end{equation}

Agent $i$'s value function estimate is updated using the noisy returns and results in the estimate $V_{i,k+1}$ at the end of iteration $k$.
We summarize the simulation-based $\texttt{KLC-OPI}$ scheme run by each agent $i \in \mathcal{N}$ in the following:
\begin{align} \label{scheme:kl-opi}
\texttt{KLC-OPI} \begin{cases}
  \pi^{(i)}_{P_0,k+1} &= \mathcal{G}(V_{i,k}), \\
  V_{i,k+1} &= (I - \mathrm{A}_k) V_{i,k}  + \mathrm{A}_k\Big((\operator^{\pi^{(i)}_{P_0,k+1}})^m V_{i,k} + \epsilon_{m,k} \Big),
\end{cases}
\end{align}
where $I$ is the $|\mathcal{S}| \times |\mathcal{S}|$ identity matrix, $\mathrm{A}_k$ is a $|\mathcal{S}| \times |\mathcal{S}|$ diagonal matrix with joint state learning rates $\alpha_k(s) \in \mathbb{R}^+$ as its elements, and the rollout value is $m \in \mathbb{N}^+$.
Note that $m=1$ corresponds to value iteration with KL control cost $\texttt{KLC-VI}$, and letting $m \rightarrow \infty$ gives a policy iteration with KL control cost $\texttt{KLC-PI}$ rendition of the scheme.
In addition, the learning rates' matrix $A_k$ is the same for all agents.

Similar to [Appendix 1.2 in \cite{todorovpnas}] for the single-agent linearly solvable MDP case, we show that the optimal joint policy \eqref{equation:greedy-step} follows a Boltzmann distribution by applying the Cole-Hopf transformation to the Bellman equation with $\operator^{\pi^{(i)}_{P_0,k+1}}$.
\begin{lemma} \label{lemma:joint_policy}
At iteration $k$ and for any $i \in \mathcal{N}$, the joint policy $\pi^{(i)}_{P_0,k+1}$ given $P_0$ that minimizes the discounted return of the policy evaluation step in \eqref{scheme:kl-opi} follows a Boltzmann distribution 
\begin{align} \label{equation:optimal-joint-policy}
\pi^{(i)}_{P_0,k+1}(s' | s) = \frac{P_0(s' | s)(Z_{i,k}(s'))^\gamma}{\sum\limits_{s' \in \mathcal{S}} P_0(s' | s)(Z_{i,k}(s'))^\gamma},
\end{align}
where $Z_{i,k}$ for agent $i$ is the Cole-Hopf transformation of the value function such that $Z_{i,k}(s) = e^{-V_{i,k}(s)}$ state-wise in iteration $k$. 
\end{lemma}
\begin{proof}
We have state-wise
\begin{align}
\operator V_{i,k}(s) &= \min\limits_{\pi^{(i)}_{P_0, k+1} \in \Pi} \{ C(s) + \mathop{\mathbb{E}}\limits_{s' \sim \pi^{(i)}_{P_0, k+1}(\cdot | s)} \Big[ \ln\frac{\pi^{(i)}_{P_0, k+1}(\cdot | s)}{P_0(\cdot | s)}\Big] + \mathop{\mathbb{E}}\limits_{s' \sim \pi^{(i)}_{P_0, k+1}(\cdot | s)} \Big[\gamma V_{i,k}(\cdot) \Big] \} \nonumber \\
&= \min\limits_{\pi^{(i)}_{P_0, k+1} \in \Pi}\{C(s) + \mathop{\mathbb{E}}\limits_{s' \sim \pi^{(i)}_{P_0, k+1}(\cdot | s)} \Big[ \ln\frac{\pi^{(i)}_{P_0, k+1}(\cdot | s)}{P_0(\cdot | s)}\Big] + \mathop{\mathbb{E}}\limits_{s' \sim \pi^{(i)}_{P_0, k+1}(\cdot | s)} \Big[ \ln\frac{1}{(Z_{i,k}(\cdot))^\gamma} \Big]\} \nonumber \\ 
&= \min\limits_{\pi^{(i)}_{P_0, k+1} \in \Pi}\{C(s) +  \mathop{\mathbb{E}}\limits_{s' \sim \pi^{(i)}_{P_0, k+1}(\cdot | s)} \Big[ \ln\frac{ \pi^{(i)}_{P_0, k+1}(\cdot | s)}{P_0(\cdot | s) (Z_{i,k}(\cdot))^\gamma}\Big]\}.
\end{align}
Define the constant $$d_{i,P_0, k}(s; \gamma) := \sum\limits_{s' \in \mathcal{S}} P_0(s' | s)(Z_{i,k}(s'))^\gamma.$$ 
The value function under the joint policy that achieves the minimum discounted cost becomes 
\begin{align} \label{equation:lemma-proof-kl-cost}
    \operator V_{i,k}(s) &= \min\limits_{\pi^{(i)}_{P_0, k+1} \in \Pi} \{C(s)  + \mathop{\mathbb{E}}\limits_{s' \sim \pi^{(i)}_{P_0, k+1}(\cdot | s)} \Big[ \ln\frac{\pi^{(i)}_{P_0, k+1}(\cdot | s)}{\frac{P_0(\cdot | s)(Z_{i,k}(\cdot))^\gamma}{d_{i,P_0,k}(s; \gamma)}}\Big] - \ln d_{i,P_0,k}(s; \gamma)\} \nonumber \\
    & = \min\limits_{\pi^{(i)}_{P_0, k+1} \in \Pi} \{C(s) - \ln d_{i,P_0,k}(s; \gamma) +  D_{KL}\Big(\pi^{(i)}_{P_0, k+1}(\cdot | s) || \frac{P_0(\cdot | s)(Z_{i,k}(\cdot))^\gamma}{d_{i,P_0,k}(s; \gamma)}\Big) \}.
\end{align}
Note that the first two terms ($C(s)$ and $d_{i,P_0,k}(s; \gamma)$) do not depend on $\pi^{(i)}_{P_0,k+1}$. 
Thus, the minimum is achieved when the last term, the KL cost, in \eqref{equation:lemma-proof-kl-cost} is equal to zero.
The optimal joint policy calculated by each agent $i \in \mathcal{N}$ is then given by \eqref{equation:optimal-joint-policy}.
\end{proof}
\begin{remark} We let the initial value estimates be the same for all agents, $V_{i,0}=V_{j,0}$ for all $i, j \in \mathcal{N}$. Given the initialization, the value function estimate $V_{i,k}$ for agent $i$ and thus the joint policy $\pi^{(i)}_{P_0,k}$ computed using \eqref{equation:optimal-joint-policy} is identical to other agents' estimates for all the iterations.
\end{remark}

\begin{remark}
    The noise value $\epsilon_{m,k}$ at iteration $k$ in \eqref{eq_estimate} captures two sources of error: the error from simulating only a single trajectory per joint state instead of averaging over an infinite number of trajectories and 
the estimation error term incurred from the estimated value function of the last visited joint state in the simulated trajectory.
\end{remark}

\subsection{Asymptotic Convergence}

Our main result is that each agent's value function estimates $V_{i,k}$ and individually calculated joint policy $\pi^{(i)}_{P_0, k}$ converge to the optimal value function and an optimal joint policy respectively under the \texttt{KLC-OPI} iterations.

\begin{theorem}
\label{theorem:global-error-convergence} Assume the initial value function estimate is such that $\operator V_{i,0} -  V_{i,0} \leq \mathbf{0}$ for each agent $i \in \mathcal{N}$ and $\alpha_k(s) = \mathcal{O}(\frac{1}{k})$ state-wise. The value function estimate $V_{i,k}$ and joint policy $\pi^{(i)}_{P_0,k}$ iterates of $i\in \mathcal{N}$ for the $\texttt{KLC-OPI}$ scheme in \eqref{scheme:kl-opi} asymptotically converge to $V^*$ and to $\pi^{*}_{P_0}$, respectively.
\end{theorem}
The convergence to optimal value function builds on two main assumptions. First one is that the initialized value function estimates $i \in \mathcal{N}$ satisfy $\operator V_{i,k=0} - V_{i,k=0} \leq \mathbf{0}$.
The  assumption is not restrictive since the initial value function estimate can be set to a large value state-wise such that a single application of the KL optimal Bellman operator will guarantee a lower value function estimate. Our second assumption on the step size is standard in stochastic approximation and ensures that the steps are square summable but not summable \cite{tsitsiklis2002,winnicki2023convergence}.

Similar to optimistic policy iteration with a deterministic policy [Proposition 1 in \cite{scherrer15}], applying the operator $\operator^{\pi^{(i)}_{P_0,k+1}}$ $m>1$ times in a single iteration of $\texttt{KLC-OPI}$ does not guarantee a contraction in any norm nor monotonic improvement of the value function estimate, unlike exact policy iteration schemes where the monotonicity property is preserved \cite{bertsekas1996neuro, munos2003error}. Hence, the $\texttt{KLC-OPI}$ simulation-based policy evaluation step may result in a worse performing value function estimate, i.e., a larger discounted cost return, compared to the previous iteration's estimate. 

Given the non-contracting, and in general without additional assumptions, non-monotonic improvement challenges of optimistic policy iteration schemes' updates, the updated value functions' performance has to be bounded using an analysis technique that does not utilize the contraction property.
To prove asymptotic convergence to the optimal value function, we show that the scheme's iterative updates are upper and lower bounded by two standard stochastic approximation processes that we later show both converge to the optimal value function for each agent.

We begin by defining the filtration that captures the algorithm's history up to and before the noise $\epsilon_{m,k}$ at the end of iteration $k$ is realized,
\begin{align}
    \mathcal{F}_{m,k} := \{(\epsilon_{m,k'})_{k' \leq k-1}\}. 
\end{align}
Given the filtration, it follows that the noise term has zero mean, $\mathbb{E}{[\epsilon_{m,k} | \mathcal{F}_{m,k}]} = \mathbf{0}$, and bounded variance $\mathbb{E}[||\epsilon_{m,k}||_\infty | \mathcal{F}_{m,k}] \leq \frac{q_{max} \cdot \overline{\mathbf{e}}}{1 - \gamma}$ state-wise.

Next, we state the upper bound on the policy evaluation step performance in a single iteration which is a restatement of  
Lemma 2.c in~\cite{tsitsiklis2002} with our notation.

\begin{lemma} [Lemma 2.c in~\cite{tsitsiklis2002}]\label{lemma:single-step-policy-evaluation} For the policy evaluation step at iteration $k$ in the $\texttt{KLC-OPI}$ scheme, and given the estimate of agent $i \in \mathcal{N}$ value function $V_{i,k}$, we have $(\operator^{\pi^{(i)}_{P_0,k+1}})^m V_{i,k} \leq V_{i,k} + \frac{r_{i,k}\cdot\overline{\mathbf{e}}}{1 - \gamma} $, where $r_{i,k} := \max\limits_{s\in\mathcal{S}}\Big[ \operator^{\pi^{(i)}_{P_0,k+1}} V_{i,k}(s) - V_{i,k}(s) \Big]$ and $\overline{\mathbf{e}}$ is an all-ones vector.
\end{lemma}

We now prove that the $\texttt{KLC-OPI}$ policy evaluation step output in \eqref{scheme:kl-opi} is upper bounded by a standard stochastic approximation process which results in the asymptotic improvement property $\limsup\limits_{k \rightarrow \infty}\operator V_{i,k} - V_{i,k} \leq \mathbf{0}$ for each agent $i \in \mathcal{N}$.

\begin{proposition} \label{theorem:lim-sup-bound}
If $\operator V_{i,k=0}-V_{i,k=0}\leq \mathbf{0}$ and the learning rate state-wise is $\alpha_k(s) = \mathcal{O}(\frac{1}{k})$ such that $0 \leq \alpha_k(s) \leq 1$ for any iteration $k$, then the policy evaluation step for any $m\in \mathbb{N}^+$ in \eqref{scheme:kl-opi} gives 
\begin{equation} \label{eq_asymptotic_improvement}
\limsup\limits_{k \rightarrow \infty}  \operator V_{i,k} - V_{i,k} \leq \mathbf{0}.
\end{equation}
\end{proposition}
\begin{proof}
Note that $\operator^{\pi^{(i)}_{P_0,k+1}}$ is an affine transformation such that $\operator^{\pi^{(i)}_{P_0,k+1}} V_{i,k} = q_i^{\pi^{(i)}_{P_0,k+1}} + \gamma \pi^{(i)}_{P_0,k+1} V_{i,k}$.
We then have the following for $k=1$, using the policy evaluation update rule in \eqref{scheme:kl-opi},

\begin{align}
\operator^{\pi^{(i)}_{P_0,1}} V_{i,1} &= q_i^{\pi^{(i)}_{P_0,1}}  +\gamma \pi^{(i)}_{P_0,1}\Big[(I - \mathrm{A}_{0}) V_{i,0} + \mathrm{A}_{0} \Big( (\operator^{\pi^{(i)}_{P_0,1}})^m V_{i,0} + \epsilon_{m,0} \Big)  \Big] \nonumber \\
&= q_i^{\pi^{(i)}_{P_0,1}} + \gamma \pi^{(i)}_{P_0,1} V_{i,0} - \mathrm{A}_{0} \gamma \pi^{(i)}_{P_0,1}V_{i,0} + \mathrm{A}_{0}\gamma \pi^{(i)}_{P_0,1}  (\operator^{\pi^{(i)}_{P_0,1}})^m V_{i,0} + \mathrm{A}_{0} \gamma \pi^{(i)}_{P_0,1}\epsilon_{m,0},
\end{align}
which we rewrite as follows using the fact that $(\operator^{\pi^{(i)}_{P_0,1}})^{m+1}V_{i,0}=\operator^{\pi^{(i)}_{P_0,1}}(\operator^{\pi^{(i)}_{P_0,1}})^m V_{i,0} = q_i^{\pi^{(i)}_{P_0,1}} + \gamma \pi^{(i)}_{P_0,1} (\operator^{\pi^{(i)}_{P_0,1}})^m V_{i,0}$, and by reorganizing the terms,  
\begin{align} \label{eq_TP1V1}
\operator^{\pi^{(i)}_{P_0,1}} V_{i,1} &=(I - \mathrm{A}_{0}) (\operator^{\pi^{(i)}_{P_0,1}} V_{i,0}) + \mathrm{A}_{0} \Big((\operator^{\pi^{(i)}_{P_0,1}})^{m+1} V_{i,0} + \gamma \pi^{(i)}_{P_0,1} \epsilon_{m,0} \Big). 
\end{align}
We know that the optimal Bellman operator is going to be better, i.e., $\operator V_{i,1} \leq \operator^{\pi^{(i)}_{P_0,1}} V_{i,1}$.
Let $y_{i,k}:= \operator V_{i,k} - V_{i,k}$. 
Subtracting $V_{i,1}$ from both sides of \eqref{eq_TP1V1}, and using the fact that $\operator V_{i,k} = \operator^{\pi^{(i)}_{P_0,k+1}} V_{i,k}$ due to the greedy step in \eqref{scheme:kl-opi}, we have
\begin{align}
y_{i,1} &\leq (I - \mathrm{A}_{0}) (\operator V_{i,0}) + \mathrm{A}_{0} \Big((\operator^{\pi^{(i)}_{P_0,1}})^{m+1} V_{i,0} + \gamma \pi^{(i)}_{P_0,1} \epsilon_{m,0} \Big) - (I - \mathrm{A}_{0}) V_{i,0} - \mathrm{A}_{0} \Big( (\operator^{\pi^{(i)}_{P_0,1}})^m V_{i,0} + \epsilon_{m,0} \Big) \nonumber \\ 
&= (I - \mathrm{A}_{0}) y_{i,0} + \mathrm{A}_{0} \Big[ (\gamma \pi^{(i)}_{P_0,1} - I) \epsilon_{m,0} \Big] + \mathrm{A}_{0} \Big( (\operator^{\pi^{(i)}_{P_0,1}})^{m+1} V_{i,0} - (\operator^{\pi^{(i)}_{P_0,1}})^{m} V_{i,0} \Big). \label{eq_yk1}
\end{align}
Since we assume that $\operator V_{i,0} \leq V_{i,0}$, i.e., $y_{i,0} \leq \mathbf{0}$, then applying the operator $m$ times on both sides gives $(\operator^{\pi^{(i)}_{P_0,1}})^{m+1} V_{i,0} - (\operator^{\pi^{(i)}_{P_0,1}})^{m} V_{i,0} \leq \mathbf{0}$. Thus we obtain the upper bound
\begin{align}
y_{i,1} \leq \mathrm{A}_{0} \Big[( \gamma \pi^{(i)}_{P_0,1} - I) \epsilon_{m,0}\Big]. 
\end{align}
Define $w_{k} := ( \gamma \pi^{(i)}_{P_0,k+1} - I) \epsilon_{m,k}$ where $w_k$ satisfies the same properties as $\epsilon_{m,k}$, i.e., zero-mean and bounded variance, then $y_{i,1} \leq \mathrm{A}_{0} w_{0} = U_{1}$ where $U_{1}$ is the upper bound on $y_{i,1}$.
Similarly, we have the following upper bound for $k=2$ 
\begin{align} \label{eq_y_k2_bound}
y_{i,2}& \leq (I - \mathrm{A}_{1}) y_{i,1} + \mathrm{A}_{1} w_{1} + \mathrm{A}_{1} \Big( (\operator^{\pi^{(i)}_{P_0,2}})^{m+1} V_{i,1} - (\operator^{\pi^{(i)}_{P_0,2}})^{m} V_{i,1}  \Big).
\end{align}
The upper bound on $y_{i,1}\leq U_{1}$ implies $\operator^{\pi^{(i)}_{P_0,2}} V_{i,1} \leq V_{i,1} + U_{1}$. When we apply the operator $\operator^{\pi^{(i)}_{P_0,k+1}}$ to both sides of the inequality for $k=2$, we obtain the following using the distributivity property $\operator^{\pi^{(i)}_{P_0,k+1}} (V_{i,k} + c \cdot \overline{\mathbf{e}}) = \operator^{\pi^{(i)}_{P_0,k+1}} V_{i,k} + \gamma \cdot c \cdot \overline{\mathbf{e}}$ for some constant $c\in \mathbb{R}$, 
\begin{align}
    (\operator^{\pi^{(i)}_{P_0,2}})^2 V_{i,1} &\leq (\operator^{\pi^{(i)}_{P_0,2}}) \Big( V_{i,1} + U_{1} \Big) \nonumber \\
    &= (\operator^{\pi^{(i)}_{P_0,2}}) V_{i,1} + \gamma U_{1}.
\end{align}
Applying $\operator^{\pi^{(i)}_{P_0,2}}$ $m$ times gives $(\operator^{\pi^{(i)}_{P_0,2}})^{m+1} V_{i,1} \leq (\operator^{\pi^{(i)}_{P_0,2}})^m V_{i,1} + \gamma^m U_{1}.$ Then, we can replace the last term in \eqref{eq_y_k2_bound} with $\gamma^m U_{1}$ to get the following bound
\begin{align}
y_{i,2} &\leq (I - \mathrm{A}_{1}) U_{1} + \mathrm{A}_{1} w_{1} + \mathrm{A}_{1} \gamma^m U_{1} \nonumber \\
&= \Big( I + \mathrm{A}_{1} (\gamma^mI - I )\Big) U_{1} + \mathrm{A}_{1} w_{1}.
\end{align}
We let $U_{2}:=( I + \mathrm{A}_{1} (\gamma^mI - I )) U_{1} + \mathrm{A}_{1} w_{1}.$
Repeating the above steps for $k = 3, 4, \hdots, K$ we obtain the upper bound for any $k$
\begin{align}
(\operator^{\pi^{(i)}_{P_0,k+1}})^{m+1} V_{i,k} - (\operator^{\pi^{(i)}_{P_0,k+1}})^m V_{i,k} &\leq \gamma^m U_{k} \nonumber \\
&= \gamma^m \Big[\Big(I + \mathrm{A}_{k-1}(\gamma^mI - I)\Big) U_{k-1}+ \mathrm{A}_{k-1} w_{k-1} \Big],
\end{align}
with $U_{0} = \mathbf{0}$. 
Since the noise is zero-mean with bounded variance, $\gamma \leq 1$, and $\alpha_k(s) = \mathcal{O}(1/k)$ with $0 \leq \alpha_k(s) \leq 1$ state-wise for any $k$, we have that 
\begin{align} \label{equation:limit}
 &\lim\limits_{k \rightarrow \infty} \alpha_k \Big[ (\operator^{\pi^{(i)}_{P_0,k+1}})^{m+1} V_{i,k} - (\operator^{\pi^{(i)}_{P_0,k+1}})^m V_{i,k}\Big] \leq \gamma^m \lim\limits_{k \rightarrow \infty} \frac{U_{k}}{k} = \mathbf{0}.
\end{align}

Define a noise function as $f(w_k, \gamma^m, U_{k}) := w_k + \gamma^m U_{k}$, then we have for any $i \in \mathcal{N}$ and $k$
\begin{align}
y_{i,k+1} \leq (I - \mathrm{A}_k) y_{i,k} + \mathrm{A}_k f(w_k, \gamma^m, U_{k}).
\end{align}
Define another sequence $X_{k+1} = (I - \mathrm{A}_k) X_{k} + \mathrm{A}_k f(w_k, \gamma^m, U_{k})$ with $X_{0} = y_{i,0}$ for any $i \in \mathcal{N}$. We have that $y_{i,k} \leq X_k$.
Given \eqref{equation:limit} and $f(w_k, \gamma^m, U_{k})$ is a zero-mean noise function, the sequence $\{X_k\}_{0,1,\dots}$ is a standard stochastic  sequence that converges to zero in the limit.
Finally, we have $\lim\limits_{k \rightarrow \infty} X_k = \mathbf{0}$ and since $y_{i,k} \leq X_k$, we obtain $\limsup\limits_{k \rightarrow \infty} y_{i,k} \leq \mathbf{0}$.
\end{proof}

Establishing the asymptotic policy improvement property of \eqref{scheme:kl-opi} is the key contribution of our proof approach. This result is similar to the property obtained in \cite{tsitsiklis2002} for showing convergence of OPI with infinitely long single trajectories. However, here we consider finitely long trajectories which is in agreement with practical implementation. In the proof, we overcome the error introduced by finite-trajectories by upper bounding the noisy and finite rollout with a term that only depends on the discount factor, learning rates, the joint policy from the previous iteration, and the realized noise up to iteration $k$.
Moreover, we achieve this property using greedy policy improvements \eqref{equation:greedy-step}, i.e., without the need to perform a lookahead operation for each iteration as is done in \cite{winnicki2023convergence, winnicki2021role}.

We are ready to provide the proof of  Theorem~\ref{theorem:global-error-convergence}. 

\begin{proof}[Proof (Theorem~\ref{theorem:global-error-convergence})]
Given the asymptotic improvement guarantee in \eqref{eq_asymptotic_improvement}, 
for every $\delta > 0$, there exists an iteration $k(\delta)$ such that $r_{i,k} \leq \delta$ for $k \geq k(\delta)$ where we recall that $r_{i,k} = \max\limits_{s\in\mathcal{S}} [\operator^{\pi^{(i)}_{P_0, k+1}} V_{i,k} - V_{i,k}]$.
Using the result from Lemma~\ref{lemma:single-step-policy-evaluation}, we have $(\operator^{\pi^{(i)}_{P_0,k+1}})^{m-1} V_{i,k} \leq V_{i,k} + \frac{r_{i,k} \cdot \overline{\mathbf{e}}}{1 - \gamma}$.
Then applying the operator on both sides gives $(\operator^{\pi^{(i)}_{P_0,k+1}})^{m} V_{i,k} \leq \operator^{\pi^{(i)}_{P_0,k+1}} V_{i,k} + \frac{\gamma \cdot r_{i,k} \cdot \overline{\mathbf{e}}}{1 - \gamma} \leq \operator^{\pi^{(i)}_{P_0,k+1}} V_{i,k} + \frac{\gamma \cdot \delta \cdot \overline{\mathbf{e}}}{1 - \gamma}$ for $k \geq k(\delta)$.

We then write the update rule from \eqref{scheme:kl-opi} for $k \geq k(\delta)$ as
\begin{align}
    V_{i,k+1} &= (I - \mathrm{A}_k) V_{i,k} + \mathrm{A}_k \Big( (\operator^{\pi^{(i)}_{P_0,k+1}})^m V_{i,k} + \epsilon_{m,k} \Big) \nonumber \\
    &\leq (I - \mathrm{A}_k) V_{i,k} + \mathrm{A}_k \Big( \operator^{\pi^{(i)}_{P_0,k+1}} V_{i,k} + \frac{\gamma \cdot \delta \cdot \overline{\mathbf{e}}}{1 - \gamma} + \epsilon_{m,k} \Big).
\end{align}
For every positive integer $\bar k$, define the random variable sequence 
\begin{align} \label{equation:second-sequence}
W_{i,k+1,\delta}^{(\bar k)} = (I - \mathrm{A}_k) W_{i,k,\delta}^{(\bar k)} + \mathrm{A}_k \Big(\operator^{\pi^{(i)}_{P_0,k+1}} W_{i,k,\delta}^{(\bar k)} + \frac{\gamma \cdot \delta \cdot \overline{\mathbf{e}}}{1 - \gamma} + \epsilon_{m,k} \Big),
\end{align} 
for all $k \geq \bar k$.
In addition, let $A^{(\bar k)}$ be the event that $k(\delta) = \bar k$.
Then for any $\bar k$, any sample path in $A^{(\bar k)}$, and for all $k \geq \bar k$, we have $V_{i,k+1} \leq W^{(\bar k)}_{i,k+1,\delta}$.
Next, define the mappings $\mathcal{H} : \mathbb{R}^{|\mathcal{S}|} \rightarrow \mathbb{R}^{|\mathcal{S}|}$ and $\mathcal{H}^{\pi^{(i)}_{P_0, k+1}} : \mathbb{R}^{|\mathcal{S}|} \rightarrow \mathbb{R}^{|\mathcal{S}|}$ such that  $\mathcal{H}^{\pi^{(i)}_{P_0, k+1}} V_{i,k} = \operator^{\pi^{(i)}_{P_0,k+1}} V_{i,k} + \frac{\gamma \cdot \delta \cdot \overline{\mathbf{e}}}{1 - \gamma}$.
Rewriting the sequence in Equation~\eqref{equation:second-sequence} using $\mathcal{H}^{\pi^{(i)}_{P_0,k+1}}$ gives
\begin{equation} \label{equation:third-sequence}
W_{i,k+1,\delta}^{(\bar k)} = (I - \mathrm{A}_k) W_{i,k,\delta}^{(\bar k)} + \mathrm{A}_k \Big(\mathcal{H}^{\pi^{(i)}_{P_0,k+1}}W_{i,k,\delta}^{(\bar k)} + \epsilon_{m,k} \Big).
\end{equation}

Since a single application of $\operator^{\pi^{(i)}_{P_0,k+1}}$ is a contraction mapping, a single application of $\mathcal{H}^{\pi^{(i)}_{P_0, k+1}}$ is also a contraction mapping with a unique fixed point $W^*_{\delta} = V^* + \frac{\gamma \cdot \delta \cdot \overline{\mathbf{e}}}{(1 - \gamma)^2}$.
Given that $\epsilon_{m,k}$ is a zero-mean noise, the sequence in Equation~\eqref{equation:third-sequence} converges almost surely to $W^*_{\delta}$ for sample paths in $A^{(\bar k)}$.
Since the union of the events $A^{(\bar k)}$ is the entire sample space, we have $\limsup\limits_{k \rightarrow \infty} V_{i,k} \leq \limsup\limits_{k \rightarrow \infty} W_{i,k,\delta}^{(\bar k)} = W_{\delta}^*$.
Given that $\delta$ can be chosen arbitrarily close to zero, then $\limsup\limits_{k \rightarrow \infty} V_{i,k} \leq V^*$.
Similarly, since $(\operator^{\pi^{(i)}_{P_0,k+1}})^m V_{i,k} \geq V^*$, we have 
\begin{align}
    V_{i,k+1} &= (I - \mathrm{A}_k) V_{i,k} + \mathrm{A}_k \Big((\operator^{\pi^{(i)}_{P_0,k+1}})^{m} V_{i,k}  + \epsilon_{m,k} \Big) \nonumber \\
    &\geq (I - \mathrm{A}_k) V_{i,k} + \mathrm{A}_k \Big(V^*  + \epsilon_{m,k} \Big).
\end{align}
Define another random variable sequence $B_{i,k+1} = (I - \mathrm{A}_k) B_{i,k} + \mathrm{A}_k \Big(V^*  + \epsilon_{m,k} \Big)$ with $B_{i,k=0} = V_{i,k=0}$.
The sequence $B_{i,k}$ converges asymptotically to $V^*$ and we have $\liminf\limits_{k \rightarrow \infty} V_{i,k} \geq V^*$.
Given the asymptotic upper and lower convergence bounds to $V^*$, the update rule in \eqref{scheme:kl-opi} also converges to the optimal value function $V^*$ for each $i \in \mathcal{N}$ and we obtain an optimal joint policy $\pi^{*}_{P_0}$ using \eqref{equation:optimal-joint-policy}.
\end{proof}
Note that we do not use any structural feature of multi-agent MDPs with KL control cost in the proof, and thus the convergence result for Theorem \ref{theorem:global-error-convergence} holds for general MDPs. 

\begin{figure*}[ht]
\begin{center} 
\subfigure[$\texttt{ASYNC-KLC-OPI}$ training performance.]{\includegraphics[width=0.32\textwidth, height=1.8in]{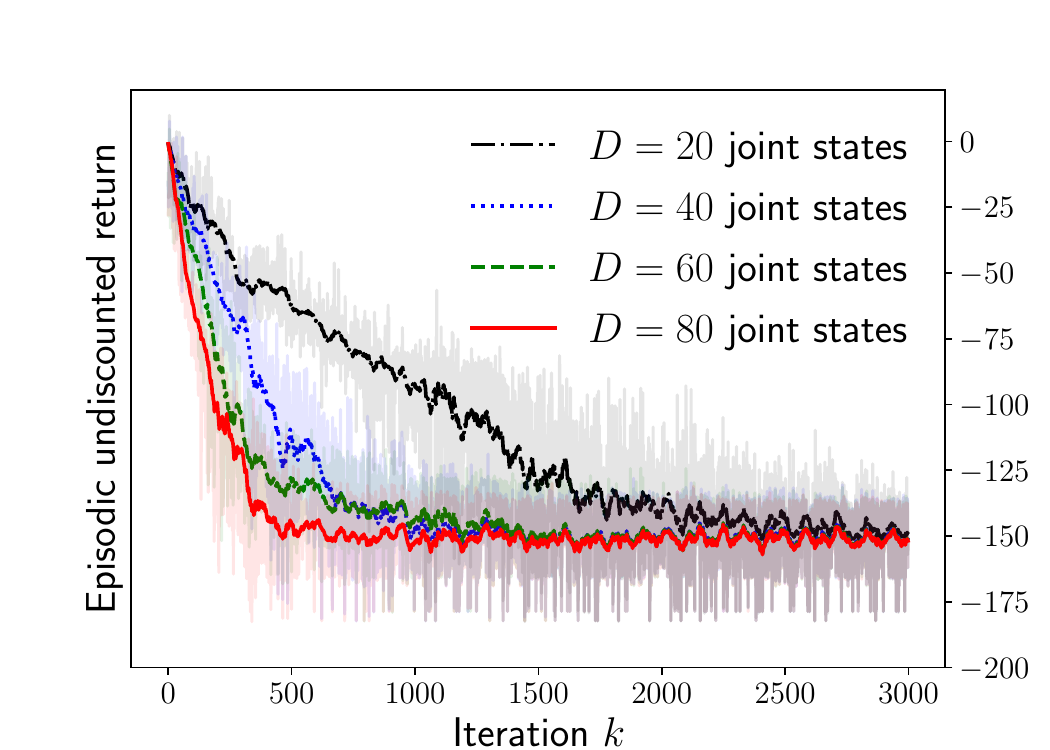}\label{fig:identical-interest-training}}
\subfigure[Comparison against the optimal deterministic joint policy with $D=80$.]{\includegraphics[width=0.32\textwidth, height=1.8in]{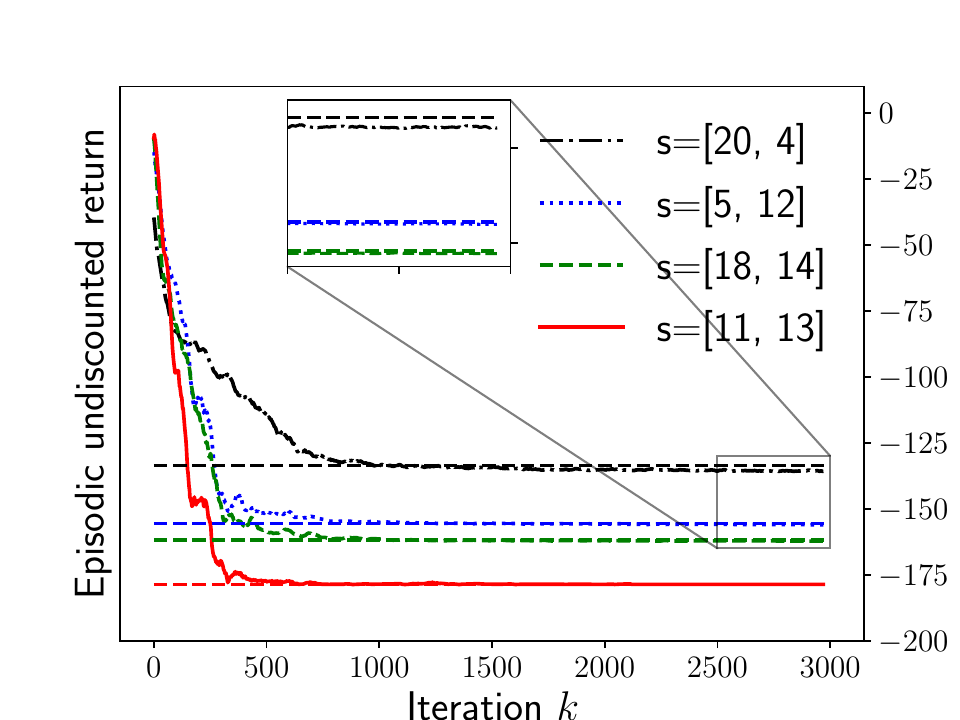}\label{fig:identical-interest-evaluation-80}} 
\subfigure[$||V^k_{i=1} -  V_{i=1}^K||_\infty$.]{\includegraphics[width=0.32\textwidth, height=1.8in]{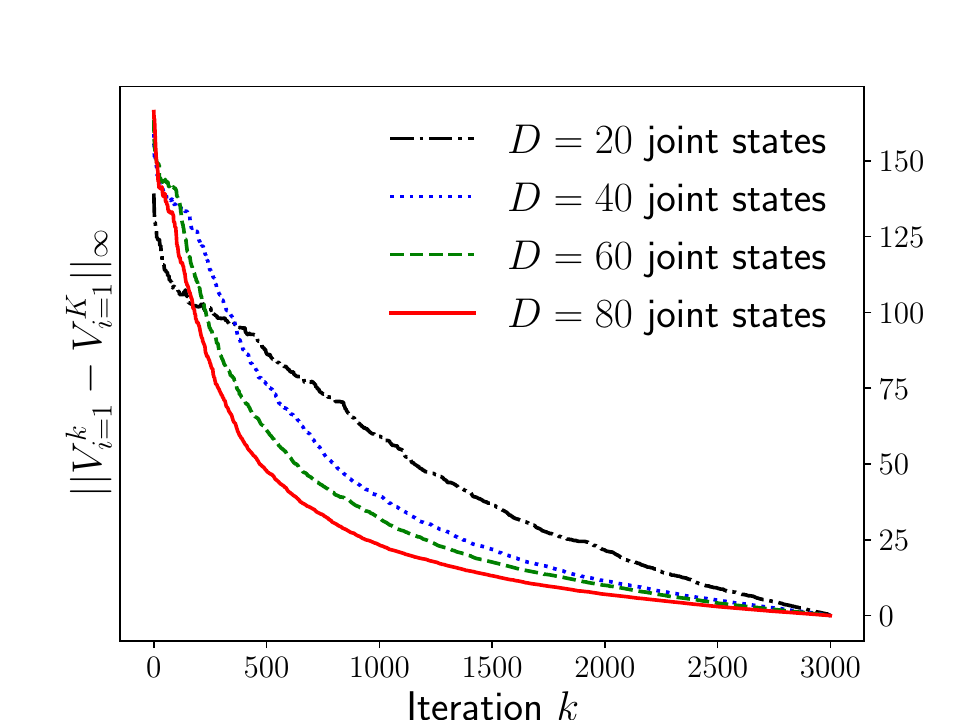}\label{fig:identical-interest-norm}} 
\end{center}
\caption{$\texttt{ASYNC-KLC-OPI}$ performance on the multi-agent MDP with KL control cost variant of the Stag-Hare game averaged over $10$ simulation runs.}
\label{fig:identical-interest-main-text}
\end{figure*}

\section{The Asynchronous $\texttt{KLC-OPI}$ Scheme} \label{sec:async-klc-opi}

For simulation-based policy iteration schemes, it can be computationally expensive to run a synchronous policy evaluation step with $|\mathcal{S}|$ $m-$step TD sampled trajectories especially if the joint state space is large.
For this reason, we focus on the asynchronous policy evaluation step of value function estimates using $1 \leq D \leq |\mathcal{S}|$ joint states in each iteration $k$.
In a single iteration, a larger $D$ value means that a larger subspace of the joint state space is evaluated at the expense of additional sampling per iteration.

In order to prove the convergence of the asynchronous version of the scheme, we assume that the initial state distribution $\rho$ is a uniform distribution.
Let $\mathcal{D}_k \subseteq \mathcal{S}$ be the set of joint states that are evaluated at iteration $k$ under the current joint policy $\pi^{(i)}_{P_0,k+1}$ for any $i \in \mathcal{S}$.
The set $\mathcal{D}_k$ of each iteration satisfies the following assumption.
\begin{assumption} \label{assumption:fixed-cardinality}
    The cardinality of the sampled joint states' set  $D = |\mathcal{D}_k|$ is a constant with the set having unique joint state elements in each iteration $k=0,1,2,\hdots, K$.
\end{assumption}
Assumption~\ref{assumption:fixed-cardinality} ensures that no joint state is evaluated more than once per iteration, and that the number of evaluated joint states is fixed in the proposed asynchronous policy evaluation step.

Similar to the synchronous policy evaluation step in~\eqref{scheme:kl-opi}, each joint state $s \in \mathcal{D}_k$ is evaluated using its independently simulated $m-$step TD trajectory and its corresponding value function estimate is updated at the end of the current iteration. 
For joint states that are not in the set $\mathcal{D}_k$, their value function estimates are kept the same.
The asynchronous $\texttt{KLC-OPI}$ scheme is then 
\begin{align} \label{scheme:async-kl-opi}
\texttt{ASYNC-KLC-OPI} \begin{cases}
  \pi^{(i)}_{P_0,k+1} &= \mathcal{G}(V_{i,k}), \\
  V_{i,k+1}(s) &= \big(1 - \alpha_k(s)\big) V_{i,k}(s) + \alpha_k(s)\Big((\operator^{\pi^{(i)}_{P_0,k+1}})^m V_{i,k}(s) + \epsilon_{m,k}(s) \Big) \hspace{1,5em} \forall s \in \mathcal{D}_k, \\
  V_{i,k+1}(s) &= V_{i,k}(s) \hfill \forall s \notin \mathcal{D}_k.
\end{cases}
\end{align}
Note that the policy improvement step taken by each agent is with respect to all joint state elements of the current value function estimate regardless of whether a joint state's value estimate was evaluated or not. 
We can re-write the policy evaluation step in~\eqref{scheme:async-kl-opi} as follows
\begin{align}
    V_{i,k+1}(s) = &\mathcal{I}_{s \in \mathcal{D}_k} \Big[\big(1 - \alpha_k(s)\big)V_{i,k}(s) + \alpha_k(s) \big[(\operator^{\pi^{(i)}_{P_0,k+1}})^m V_{i,k}(s) + \epsilon_{m,k}(s)\big]\Big] + \mathcal{I}_{s \notin \mathcal{D}_k} \Big[ V_{i,k}(s)\Big],
\end{align}
where $\mathcal{I}_{s \in \mathcal{D}_k}$ is the indicator function that specifies if the joint state $s$ is evaluated in iteration $k$ or not.
In the following, we establish the same asymptotic convergence guarantee as in Theorem~\ref{theorem:global-error-convergence} for the iterates in \texttt{ASYNC-KLC-OPI}. 
\begin{corollary} \label{corollary:async-klc-opi-convergence}
Assume that the initial value function is such that $\operator V_{i,0} - V_{i,0} \leq \mathbf{0}$ for each agent $i \in \mathcal{N}$, $\alpha_k(s) = \mathcal{O}(\frac{1}{k})$ state-wise, and that the initial joint state distribution $\rho$ is a uniform distribution, then the $\texttt{ASYNC-KLC-OPI}$ scheme in~\eqref{scheme:async-kl-opi} value function $V_{i,k}$ and joint policy $\pi^{(i)}_{P_0,k}$ iterates asymptotically converge to $V^*$ and to $\pi^{*}_{P_0}$ for all $i\in \mathcal{N}$, respectively. 
\end{corollary}
\vspace{-1em}
\begin{proof}
We can write the policy evaluation update rule of the asynchronous scheme in \eqref{scheme:async-kl-opi} similar fashion to the synchronous case \eqref{scheme:kl-opi}, 
\begin{align}
    V_{i,k+1} = (I - \mathrm{A}_k H)V_{i,k} + \mathrm{A}_k H \Big( (\operator^{\pi^{(i)}_{P_0,k+1}})^m V_{i,k} + g_{m,k} \Big),
\end{align}
with $H$ being the $|\mathcal{S}| \times |\mathcal{S}|$ diagonal joint state policy evaluation probability matrix, i.e. the matrix $H$ element $h(s) > 0$ state-wise is the stationary probability that the joint state $s$ is sampled according to.
The added noise function $g_{m,k}$ is given by
\begin{align} \label{equation:g-noise-function}
    g_{m,k} := \quad & \epsilon_{m,k} + \Big((H^{-1}\mathcal{X}_k - I) \cdot (-V_{i,k} + (\operator^{\pi^{(i)}_{P_0,k+1}})^m V_{i,k} + \epsilon_{m,k}) \Big),
\end{align}
with $\mathcal{X}_k$ being a $|\mathcal{S}| \times |\mathcal{S}|$ diagonal matrix with Bernoulli random variable elements $x_k(s)$ state-wise at iteration $k$.
The random variable $x_k(s) = 1$ if the joint state $s$ is evaluated during iteration $k$, and it is zero otherwise.
Since we assume that the initial state distribution $\rho$ is a uniform distribution in the asynchronous case, the expected value of the Bernoulli random variables is $\mathbb{E}[ x_k(s) | \mathcal{F}_{m,k}] = h(s)$ state-wise.
We then have $H^{-1} \cdot \mathbb{E}[\mathcal{X}_k | \mathcal{F}_{m,k}] = H^{-1} \cdot H = I$, which means that $g_{m,k}$ is also a zero-mean noise function such that $\mathbb{E} [g_{m,k} | \mathcal{F}_{m,k}] = \mathbf{0}$.
Similar to the noise function $\epsilon_{m,k}$, the added noise function $g_{m,k}$ variance is also bounded since the value function estimate and $(\operator^{\pi^{(i)}_{P_0,k+1}})^m V_{i,k}$ in Equation~\eqref{equation:g-noise-function} are bounded by the value $\frac{q_{max}}{1 - \gamma}$ state-wise.

Given Assumption~\ref{assumption:fixed-cardinality}, the set's cardinality $|\mathcal{D}_k|$ is a constant throughout the iterations, then the diagonal matrix $H$ has equal and constant elements throughout the iterations $k=0,1,2,\hdots, K$. 
In other words, the matrix elements $h(s) = h(s')$ for all $(s,s') \in \mathcal{S}$ pairs.
Define $\mathcal{A}_k = \mathrm{A}_k \cdot H$ as the product of the diagonal learning rate matrix $\mathrm{A}_k$ and the diagonal joint state policy evaluation probability matrix $H$.
We can then write the policy evaluation step as 
\begin{align}
    V_{i,k+1} = \big( I - \mathcal{A}_k\big)V_{i,k} + \mathcal{A}_k \big( (\operator^{\pi^{(i)}_{P_0,k+1}})^m V_{i,k} + g_{m,k}\big).
\end{align}
Since $H$ has constant elements throughout the iterations, then $\mathcal{A}_k$ is a matrix that commutes with the joint policy $\pi^{(i)}_{P_0,k+1}$ in each iteration such that $\mathcal{A}_k \cdot \pi^{(i)}_{P_0,k+1} = \pi^{(i)}_{P_0,k+1} \cdot \mathcal{A}_k$. 
Since $\mathcal{A}_k \cdot \pi^{(i)}_{P_0,k+1}$ satisfy the commutative property and the added noise function $g_{m,k}$ also has zero mean with bounded variance, we can directly use the same steps as in Proposition~\ref{theorem:lim-sup-bound} to show that $\lim\sup\limits_{k \rightarrow \infty} \operator V_{i,k} - V_{i,k} \leq \mathbf{0}$ for each agent $i \in \mathcal{N}$.
The convergence to the optimal value function $V^*$ for each agent and to an optimal joint policy $\pi^{*}_{P_0}$ follows the same steps as in Theorem~\ref{theorem:global-error-convergence}.
\end{proof}

\section{Simulations} \label{section:simulations}

\subsection{A Multi-Agent MDP: Stag-Hare}
We consider a two-agent MDP with KL control cost variant~\cite{kappen2012optimal} of the Stag-Hare game \cite{skyrms2004stag}. 
There are $|\mathcal{N}| = 2$ hunters who hunt on a $5 \times 5$ gridworld and can only move to adjacent grids in a single timestep.

Hunter $i$'s sub-state is their grid location $s_i = 0, 1, \hdots, 24$ such that there are $|\mathcal{S}_i| = 25$ possible grid locations.
For the two hunters, this gives a total of $|\mathcal{S}| = |\mathcal{S}_1| \times |\mathcal{S}_2| = 625$ joint states.
The uncontrolled transition probability $P_{i,0}$ forces agent $i \in \mathcal{N}$ to remain in their current sub-state w.p. 0.9, and transition to one of the $b$ adjacent sub-states w.p. $0.1/b$.  

The gridworld has four hares and one stag at sub-state locations $s_h \in \{0, 4, 20, 24\}$ and $s_s \in \{12\}$, respectively. 
When one of the hunters reaches a hare's location, both hunters obtain a negative intrinsic joint state cost of $-2$ added to the KL control cost per timestep.
If both hunters cooperate and move together to the stag's location, they would obtain a lower intrinsic joint state cost of $-10$ added to the KL control cost per timestep.
The intrinsic joint state cost incurred for each agent $i \in \mathcal{N}$ in a single transition is 
\begin{align}
    C(s) = -2 \cdot \sum\limits_{i=1}^{n} \mathcal{I}\{s_i \in s_h\} - 10 \cdot \mathcal{I}\Big\{ \Big[\sum\limits_{i=1}^{n} \mathcal{I}\{s_i \in s_s\}\Big] > 1 \Big\},
\end{align}
where $\mathcal{I}\{\cdot\}$ being the indicator function.

\subsection{$\texttt{ASYNC-KLC-OPI}$ Simulation Results}
We set the discount factor to be $\gamma = 0.95$, and the fixed episode time horizon $m$ to be the average of a geometrically distributed random variable with distribution $\texttt{geom}(1-\gamma)$ such that $m=20$.
We test the scheme using $D = [20, 40, 60, 80]$ joint states for $K = 3000$ iterations and show averaged results over $10$ simulation runs in Figure~\eqref{fig:identical-interest-main-text}.

In Figure~\ref{fig:identical-interest-training}, $\texttt{ASYNC-KLC-OPI}$ converges faster to a minimum cost return as the number of sampled joint states $D$ in a single iteration $K$ increases. 
With a larger $D$ value, the scheme evaluates a larger subset of the joint state space in one iteration, and as a result it obtains a better joint policy for the next iteration compared with a smaller $D$ value.
However, the average runtime per iteration increases as $D$ increases and is $6.66, 14.57, 16.18,$ and $20.75$ seconds for the selected $D$ values, respectively.

Figure~\ref{fig:identical-interest-evaluation-80} compares the obtained stochastic policy for $D=80$ joint states against the optimal deterministic joint policy.
The optimal deterministic joint policy executes grid transitions that direct the agents to the stag's sub-state in the lowest number of time steps.
It can be seen from the figure that the stochastic joint policy gives a slightly better performance for the joint states $[20,4], [5, 12]$ and $[18, 14]$ in terms of the undiscounted cost return.
For the joint state $[11, 13]$ where each hunter starts in a cell neighboring the stag, the two policies achieve a similar performance.
The cost return difference is due to the increased KL cost when using the optimal deterministic joint policy.
Agents using the optimal deterministic joint policy transition to a selected joint state w.p. 1 which results in a higher KL cost compared with a stochastic joint policy.

For Figure~\ref{fig:identical-interest-norm}, we plot the $L_\infty$-Norm for the difference between iteration $k$ value function estimates and the value function from iteration $K=3000$.
As shown, the convergence rate to the iteration $K=3000$ value function estimate is faster with a larger $D$ value.

\section{Conclusion} \label{sec:conclusion}
We presented a synchronous and an asynchronous simulation-based optimistic policy iteration schemes for multi-agent MDPs with KL control costs that are run independently by each agent. The separation between control costs and joint state costs rendered the optimal joint policy to have a close-form solution in the form of a Boltzmann distribution that depends on the current value function estimate and uncontrolled transition probabilities. Given standard assumptions on the learning rates and the initial value function estimate, we showed the asymptotic convergence of both schemes to the optimal value function and an optimal joint policy. The convergence result applies to any simulation-based OPI scheme with finite and noisy rollout returns on any MDP. 
For different number of sampled joint states in an iteration, simulation results on a multi-agent MDP variant of the Stag-Hare game showed that the asynchronous scheme converges to a minimum cost return for the agents, with better performance than the optimal deterministic joint policy.

\newpage

\bibliographystyle{unsrtnat}  
\bibliography{references}

\end{document}